\documentclass[11pt]{article}

\usepackage[margin=1in]{geometry}
\usepackage{times}
\usepackage[numbers,compress]{natbib}

\usepackage[utf8]{inputenc}
\usepackage[T1]{fontenc}
\usepackage{hyperref}
\usepackage{url}
\usepackage{booktabs}
\usepackage{amsfonts}
\usepackage{nicefrac}
\usepackage{microtype}
\usepackage{xcolor}
\usepackage{graphicx}
\usepackage{amsmath}
\usepackage{amssymb}
\usepackage{amsthm}

\newtheorem{theorem}{Theorem}
\newtheorem{corollary}[theorem]{Corollary}

\newtheorem{definition}[theorem]{Definition}

\title{Structure of the Circular--Dyadic Convolution Error}

\author{%
  Ben Fauber and Alireza Moradzadeh \\
  NVIDIA \\
  \texttt{\{bfauber,amoradzadeh\}@nvidia.com}
}
\date{}

\begin{document}

\maketitle

\begin{abstract}
Dyadic and circular convolution can both be computed in $O(N\log N)$ time using the Hadamard transform and the FFT-computed discrete Fourier transform (DFT), respectively. The Hadamard transform is preferable for its real-valued sign flips, yet its substitution for the DFT introduces algebraic error. We present three complementary results that characterize this error. First, we identify exact error cancellation: two input and two output positions are universally error-free, and no reordering of the output can eliminate this error. Second, the error operator is nearly full rank, while its null space has only logarithmic dimension. Third, the expected error is governed by a single alignment scalar, with a closed‑form expression obtained by averaging over random filters. In general, the substitution error asymptotically doubles the output energy, except for filters in the universal zero‑error subspace, which incur no error. Collectively, these results show that the substitution error is structured, predictable, and governed by alignment.
\end{abstract}

\section{Introduction}
\label{sec:intro}

The Hadamard transform offers the same $O(N\log N)$ complexity as the fast Fourier transform (FFT) while using only real arithmetic and sign flips \cite{beauchamp1975walsh, cooley1965algorithm}, making it a favorable substitute in random projection settings. For example, the subsampled randomized Hadamard transform \cite{tropp2011improvedanalysissubsampledrandomized, doi:10.1137/100810447} achieves the same dimensionality-reduction guarantees as the fast Johnson-Lindenstrauss transform \cite{ailonchazelle2006fjlt}, and Hadamard-based kernel approximations \cite{cherapanamjeri2022uniformapproximationsrandomizedhadamard} match random Fourier features \cite{NIPS2007_013a006f} in approximation quality. In both cases, interchangeability follows from shared incoherence and concentration properties rather than algebraic structure.

In the convolution setting, algebraic structure is decisive: the discrete Fourier transform (DFT) diagonalizes \emph{circular} convolution via modular indexing, while the Hadamard transform diagonalizes \emph{dyadic} convolution via XOR indexing \cite{Terras_1999}. Substituting one for the other therefore changes the operation being computed. Hadamard layers have already been proposed as drop-in replacements for DFT-based convolution in deep neural networks~\cite{pan2021fast, pan2022block, pan2023htlayer}, yet to our knowledge no prior work has identified the signal conditions under which this error vanishes, characterized the rank and null-space structure of the error operator, or derived a closed-form expression for its expected magnitude.

Three complementary results formalize the effects of this substitution. First (\emph{zero-error conditions}): the mismatch permutation $\rho_n$ (Section~\ref{sec:general-structure}) is the identity at exactly two output positions, and exactly two input positions are fixed by $\rho_n$ for every output position $n$. These are therefore the only locations that are universally error‑free, and no output reordering can extend this agreement to all inputs. Second (\emph{operator rank and the equivalence subspace}): the orbit structure of the group generated by the mismatch permutations determines a null space of only logarithmic dimension, leaving the error operator nearly full rank. Third (\emph{error magnitude for generic signals}): a single alignment scalar governs the expected substitution error, which asymptotically doubles the output energy for generic filters, while filters in the universal zero-error subspace incur no error.

\section{Related Work}

\paragraph{Substitution error and algebraic foundations.}
P\"uschel and Moura~\cite{PueschelMoura2008_ASP_1DSpace, PueschelMoura2008_ASP_Foundation1DTime} show that circular and dyadic convolution arise from fundamentally different signal models and are therefore not algebraically interchangeable. Related structural mismatches have been studied in spectral leakage~\cite{harris1978windows} and mismatched filtering~\cite{turin1960matched}, where errors from incorrect periodicity or kernel assumptions are structured and analytically tractable. This paper applies a similar analytic lens to a different mismatch: the filter is fixed while the convolution group is changed. Although finite‑group Fourier theory~\cite{Terras_1999} establishes that $\mathbb{Z}_N$ and $(\mathbb{Z}_2)^m$ are non‑isomorphic, the rank structure, null space, and expected‑error behavior of the difference operator have not previously been characterized.

\paragraph{Structured transforms in sequence models.}
State-space models \cite{gu2022iclr-efficiently} and long-convolution architectures \cite{poli2023hyenahierarchylargerconvolutional, brixi2026genomemodellingevo2} reduce sequence-mixing from the $O(N^2)$ cost of Attention~\cite{10.5555/3295222.3295349} to $O(N\log N)$ using DFT-based circular convolution, establishing it as a practical sequence mixer at scale and motivating interest in hardware-efficient alternatives such as the Hadamard transform.

\paragraph{Hadamard layers in CNNs.}
Hadamard‑based neural network layers have been proposed as drop‑in replacements for circular‑convolution operations in convolutional neural networks (CNNs) \cite{pan2021fast, pan2022block, pan2023htlayer}, achieving competitive accuracy on image‑classification benchmarks while treating the two transforms as interchangeable.

\section{Preliminaries}

Throughout, we consider discrete real-valued signals of length $N = 2^m$ (required by the Hadamard transform; Section~\ref{sec:hadamard-transform}) for integer $m \geq 2$, where $n = 0, 1, \ldots, N-1$ is the output position index, $k = 0, 1, \ldots, N-1$ is the summation index, $u \in \mathbb{R}^N$ is the filter and $v \in \mathbb{R}^N$ is the input signal. The case $m=1$ ($N=2$) is excluded: $(\mathbb{Z}_2,+) \cong ((\mathbb{Z}_2)^1,\oplus)$, so the two convolutions coincide identically and the substitution error is trivially zero.

\subsection{Discrete Fourier Transform}

Let $\mathcal{F}$ denote the $N$-point DFT, $(\mathcal{F}\{x\})[\xi] = \sum_{j=0}^{N-1} x[j]\,e^{-2\pi \mathrm{i}\xi j/N}$, $\xi = 0,\ldots,N-1$. The DFT convolution theorem~\cite{cooley1965algorithm, lyonsdsp2011} states $\mathcal{F}\{y_{\mathrm{C}}\} = \mathcal{F}\{u\} \odot \mathcal{F}\{v\}$, where $\odot$ denotes element-wise multiplication. Applying $\mathcal{F}^{-1}$ gives $y_{\mathrm{C}} = \mathcal{F}^{-1}\!\bigl\{\mathcal{F}\{u\} \odot \mathcal{F}\{v\}\bigr\}$. Expanding, with indices taken modulo $N$:

\begin{equation}
   \label{eq:circ-conv}
   y_{\mathrm{C}}[n] = \sum_{k=0}^{N-1} u[k]\,v\bigl[(n-k)\bmod N\bigr]
\end{equation}

\subsection{Hadamard Transform}
\label{sec:hadamard-transform}

The Hadamard (Walsh--Hadamard) transform is a generalized Fourier transform~\cite{walsh1923closed, shanks1969ITCmp.100.457S, beauchamp1975walsh}. The $(k,n)$-th entry of its associated $N \times N$ operator $\mathcal{H}_m$ is:
 
 \begin{equation}                                      
 \label{eq:hadamard-entry}
   (\mathcal{H}_{m})_{k,n}={\frac {1}{2^{m/2}}}(-1)^{\sum _{r} k_{r} n_{r}}
\end{equation}

\noindent where the sum runs over bit positions $r = 0, 1, \ldots, m-1$ and $k_r, n_r \in \{0,1\}$ are the $r$-th bits of $k$ and $n$. Since $\mathcal{H}_m$ is involutory ($\mathcal{H}_{m}^2 = I_{2^m}$), $\mathcal{H}_m^{-1} = \mathcal{H}_m$. It can be computed in $O(N\log N)$ operations \cite{beauchamp1975walsh, fino1976unified}, matching FFT complexity, using only sign flips and no complex arithmetic.

\medskip
\noindent\textbf{Remark ($N = 2^m$ requirement).} The restriction to powers of two is intrinsic: the recursive butterfly structure of $\mathcal{H}_m$ requires halving at each stage, and the XOR-based index arithmetic of equation~\eqref{eq:hadamard-entry} is defined over $m$-bit integers. Extensions to non-power-of-two lengths are outside the scope of this work.

\subsection{Hadamard Transform Convolution}
\label{sec:hadamard-conv}

$\mathcal{H}_m$ supports its own convolution theorem \cite{doi:10.1049/el:19730172, beauchamp1975walsh, uvsakova2002walsh}, computing a \emph{dyadic convolution}\footnote{The term \emph{dyadic} follows the Walsh--Hadamard convolution literature~\cite{doi:10.1049/el:19730172, beauchamp1975walsh}, where it denotes XOR‑indexed convolution over $(\mathbb{Z}_2)^m$. This usage is distinct from the wavelet‑theoretic (Haar) meaning of dyadic, which refers to multi-resolution analysis over dyadic intervals~\cite{ahmed1975orthogonal}.} that is not interchangeable with the circular convolution of the DFT under any output reordering (Section~\ref{sec:alg-structures} and Theorem~\ref{thm:zero-error}). The dyadic convolution is defined as follows:

\begin{equation}
      \label{eq:dyadic-conv}
       y_{\mathrm{D}}[n] = \sum_{k=0}^{N-1} u[k]\,v[n \oplus k]
\end{equation}

\noindent where $\oplus$ denotes XOR, replacing the modular indexing $(n-k)\bmod N$ of equation~\eqref{eq:circ-conv}. The XOR indexing arises from the binary structure of equation~\eqref{eq:hadamard-entry}: the sign $(-1)^{\sum_r k_r n_r}$ decomposes independently at each bit position, and XOR is the index arithmetic that preserves this structure. The following convolution identity holds:\footnote{The $\sqrt{N}$ factor is a normalization artifact of the unitary convention $\mathcal{H}_m^2 = I$; the unnormalized matrix $H_m$ (entries $\pm1$, $H_m^{-1} = H_m/N$) gives the equivalent form $y_{\mathrm{D}} = \tfrac{1}{N}H_m\{H_m\{u\} \odot H_m\{v\}\}$, directly analogous to equation~\eqref{eq:circ-conv}.}

\begin{equation}
      \label{eq:hadamard-thm}
      y_{\mathrm{D}} = \sqrt{N}\;\mathcal{H}_m^{-1}\!\left\{\mathcal{H}_m\{u\} \odot \mathcal{H}_m\{v\}\right\}
\end{equation}

\subsection{Difference in Algebraic Structures}
\label{sec:alg-structures}

Although $\mathcal{H}_m$ is equivalent to a multidimensional DFT at the transform level~\cite{shanks1969ITCmp.100.457S, yarlagadda1997hadamard}, this equivalence does not extend to convolution over $\mathbb{Z}_N$. Both operations are group convolutions; the term \emph{circular} throughout this paper refers specifically to convolution over the cyclic group $\mathbb{Z}_N$ (modular indexing), not to group convolution in general. The DFT diagonalizes circular convolution, whereas $\mathcal{H}_m$ diagonalizes dyadic convolution over the binary group $(\mathbb{Z}_2)^m$ via XOR indexing \cite{Terras_1999}. These groups are not isomorphic for $m \geq 2$: $\mathbb{Z}_N$ contains an element of order $2^m$, while every element of $(\mathbb{Z}_2)^m$ has order at most $2$. Consequently, the two convolutions are algebraically distinct. In the framework of P\"uschel and Moura~\cite{PueschelMoura2008_ASP_1DSpace, PueschelMoura2008_ASP_Foundation1DTime}, this distinction reflects that the DFT and $\mathcal{H}_m$ are Fourier transforms for different, non‑isomorphic signal models; the structured error produced by substituting one for the other is the focus of Section~\ref{sec:results} onward.

\section{Results}
\label{sec:results}

Throughout Section~\ref{sec:results}, let $N = 2^m$ and $u, v \in \mathbb{R}^N$; we define the \emph{error signal}:

\begin{equation}
    e[n] \;\triangleq\; y_{\mathrm{C}}[n] \;-\; y_{\mathrm{D}}[n]
\end{equation}

\noindent where $y_{\mathrm{C}}$ and $y_{\mathrm{D}}$ are the circular and dyadic convolutions of equations~\eqref{eq:circ-conv} and~\eqref{eq:dyadic-conv}. Let $\mathbf{C}^u$ and $\mathbf{H}^u$ denote the $N\times N$ matrices with $(\mathbf{C}^u)_{n,k} = u[(n-k)\bmod N]$ and $(\mathbf{H}^u)_{n,k} = u[n\oplus k]$, so that $\mathbf{e} = (\mathbf{C}^u - \mathbf{H}^u)v$. The full error vector is written $\mathbf{e}$ (bold) to distinguish it from the scalar $e[n]$; $\mathbf{b}_j$ denotes the $j$-th standard basis vector. Signal vectors $u$, $v$, $y_{\mathrm{C}}$, $y_{\mathrm{D}}$ are written unbolded throughout.

\subsection{Foundational Zero-Error Conditions}
\label{sec:foundational}

The first result characterizes exact error cancellation: we identify the output and input positions where both convolutions agree for all $u, v \in \mathbb{R}^N$.

\begin{theorem}[Universal zero-error output positions]\label{thm:zero-error}
For all $N = 2^m$ with $m \geq 2$, there are exactly two output positions $n$ at which the error vanishes for all signals $u, v \in \mathbb{R}^N$:
\begin{equation}
    e[N-1] \;=\; 0 \qquad \text{and} \qquad e[N/2 - 1] \;=\; 0
\end{equation}
These are the only $n \in \{0,\ldots,N-1\}$ satisfying $(n-k)\bmod N = n \oplus k$ for all $k$; in particular, no output permutation can reconcile the two convolutions for all inputs.
\end{theorem}

\noindent\textit{Proof.}

\smallskip
\noindent\textit{(i) $e[N-1] = 0$.} Since $N-1=(1\cdots1)_2$, XOR with $N-1$ is bitwise complementation: $(N-1)\oplus k = N-1-k$, which equals $(N-1-k)\bmod N$ since $N-1-k\in\{0,\ldots,N-1\}$; thus every term of $y_{\mathrm{C}}[N-1]$ matches the corresponding term of $y_{\mathrm{D}}[N-1]$.

\smallskip
\noindent\textit{(ii) $e[N/2 - 1] = 0$.} Write $k = c_k\cdot N/2 + k'$ with $c_k\in\{0,1\}$ and $0\leq k'<N/2$. Since $n=(0\,\underbrace{1\cdots 1}_{m-1})_2$, XOR with $n$ inverts the lower $m-1$ bits and leaves the leading bit of $k$ unchanged. \textbf{Case} $c_k=0$: $n\oplus k = N/2-1-k'$ and $(n-k)\bmod N = N/2-1-k'$ (no modular reduction since $0\leq n-k < N/2$). \textbf{Case} $c_k=1$: $n\oplus k = N-1-k'$ and $(n-k)\bmod N = (-1-k')\bmod N = N-1-k'$ (since $-N<-1-k'<0$). Both cases agree.

\smallskip
\noindent\textit{(iii) Uniqueness.} We show no $n \in \{0,\ldots,N-2\}\setminus\{N/2-1\}$ satisfies the condition for all $k$. Use the witness $k = N-1$: since $(n-(N-1))\bmod N = n+1$ for $n < N-1$ (as $n-N+1<0$, adding $N$ gives $n+1$) and $n\oplus(N-1)=N-1-n$ (XOR with all-ones is bitwise complement), the condition requires $n+1 = N-1-n$, i.e.\ $n = N/2-1$. Any $n\in\{0,\ldots,N-2\}\setminus\{N/2-1\}$ therefore fails at $k=N-1$; the remaining case $n=N-1$ is already a zero-error position by part~(i). Together with parts (i) and (ii), the universal zero-error set is exactly $\{N/2-1,N-1\}$.

\smallskip
\noindent\textit{(iv) No-permutation claim.} Suppose some permutation $\phi$ satisfied $y_{\mathrm{D}}[\phi^{-1}(n)] = y_{\mathrm{C}}[n]$ for all $n$, $u$, $v$. The resulting bilinear-form equality requires $\phi^{-1}(n)\oplus k = (n-k)\bmod N$ for all $n$, $k$; setting $k=0$ gives $\phi^{-1}(n)=n$ for all $n$, so $\phi=\mathrm{id}$. But then the condition reduces to $(n-k)\bmod N = n\oplus k$ for all $n$, $k$, which part~(iii) shows fails for all $n\notin\{N/2-1,N-1\}$. $\square$

\medskip
\begin{corollary}[Dual zero-error input positions]\label{cor:dual-input}
The only $k \in \{0,\ldots,N-1\}$ satisfying $(n-k)\bmod N = n \oplus k$ for all $n$ are $k \in \{0,\,N/2\}$. In particular, columns $k=0$ and $k=N/2$ of $\mathbf{C}^u-\mathbf{H}^u$ are identically zero for all $u$.
\end{corollary}

\noindent\textit{Proof.} $k=0$ is immediate (since $(n-0)\bmod N = n = n\oplus 0$ for all $n$). For $k=N/2$: if $n \geq N/2$, then $(n-N/2)\bmod N = n-N/2$ (no wrap-around) and $n\oplus N/2 = n-N/2$ (clears the leading bit); if $n < N/2$, then $(n-N/2)\bmod N = n+N/2$ (wraps) and $n\oplus N/2 = n+N/2$ (sets the leading bit). Both cases agree. Uniqueness: for $k\neq 0$, at $n=0$, $(0-k)\bmod N = N-k$ and $0\oplus k = k$, so $N-k=k$ forces $k=N/2$. Since $k=0$ is already accounted for, these are the only two solutions. $\square$

\subsection{General Structure of the Error at an Arbitrary Position}
\label{sec:general-structure}

Having identified the positions where the error vanishes universally, we now derive a closed-form expression for $e[n]$ at an arbitrary output position $n$.

\noindent Define $\tilde{u}_n[j] \triangleq u[(n-j)\bmod N] - u[n \oplus j]$ for each output position $n$. Substituting $j=(n-k)\bmod N$ into $y_{\mathrm{C}}[n]$ and $j=n\oplus k$ into $y_{\mathrm{D}}[n]$ (both bijections on $\{0,\ldots,N-1\}$) and subtracting gives:
\begin{equation}
    \label{eq:general-closed-form}
    e[n] \;=\; \langle \tilde{u}_n,\, v \rangle \;=\; \sum_{j=0}^{N-1} \tilde{u}_n[j]\, v[j]
\end{equation}

\medskip
\begin{definition}[Mismatch permutation]\label{def:rho}
For $n \in \{0,\ldots,N-1\}$, the \emph{mismatch permutation} $\rho_n$ is:
\begin{equation}
    \label{eq:rho-def}
    \rho_n(j) \;=\; n \oplus \bigl((n-j)\bmod N\bigr)
\end{equation}
Since $\rho_n$ is the composition of $j\mapsto(n-j)\bmod N$ and $a\mapsto n\oplus a$, both bijections, it is a bijection on $\{0,\ldots,N-1\}$.
\end{definition}

\noindent The \emph{matching set} $S_n \triangleq \{j \in \{0,\ldots,N-1\} : (n-j)\bmod N = n \oplus j\}$ is the fixed-point set of $\rho_n$. Since $\tilde{u}_n[j] = 0$ whenever $j \in S_n$ (the circular and XOR indices agree there), the support of $\tilde{u}_n$ is contained in $S_n^{\mathsf{c}}$, with equality for generic $u$. Recall that $x$ is a \emph{submask} of $y$ if every set bit of $x$ is also set in $y$; the matching set has the following explicit form.

\medskip
\begin{theorem}[Matching set structure]\label{thm:matching-set}
For all $N = 2^m$ and $n \in \{0,\ldots,N-1\}$, with $n' \triangleq n \bmod (N/2)$:
\begin{equation}
    S_n \;=\; \bigl\{j \in \{0,\ldots,N-1\} \;:\; (j \bmod N/2) \text{ is a submask of } n'\bigr\}
\end{equation}
and $|S_n| = 2^{\,\mathrm{popcount}(n')\,+\,1}$, where $\mathrm{popcount}(n')$ denotes the number of $1$-bits in $n'$.
\end{theorem}

\medskip
\noindent\textit{Proof.} See Appendix~\ref{app:matching-set-proof}. $\square$

\begin{corollary}[Extremes of the matching-set spectrum]\label{cor:extremes}
$|S_n|$ is maximized at $n \in \{N/2-1,\,N-1\}$, giving $|S_n| = N$ (the universal zero-error output positions of Theorem~\ref{thm:zero-error}), and minimized at $n \in \{0,\,N/2\}$, giving $S_n = \{0,\,N/2\}$ (the universal zero-error input positions of Corollary~\ref{cor:dual-input}).
\end{corollary}

\noindent\textit{Proof.} Substitute $n' = N/2-1$ ($\mathrm{popcount}(n') = m-1$, giving $|S_n| = 2^m = N$) and $n' = 0$ ($\mathrm{popcount}(n') = 0$, giving $|S_n| = 2$) into Theorem~\ref{thm:matching-set}. $\square$

\subsection{Null-Space Structure of the Error Operator}
\label{sec:null-space}

The second result characterizes the rank and null-space structure of the error operator: the rank is generically $N-m-1$ (nearly full), and the null space has dimension $m+1$ (logarithmic in $N$). Define $\mathrm{lsb}(j) = p$ to mean $2^p \mid j$ and $2^{p+1} \nmid j$ (the position of the least significant $1$-bit of $j$). As shown in Appendix~\ref{app:rank-tightness}, the null-space dimension equals the number of orbits of $G = \langle\rho_n : n=0,\ldots,N-1\rangle$ acting on $\{0,\ldots,N-1\}$.

\begin{theorem}[Null-space structure]\label{thm:null-space}
For all $u \in \mathbb{R}^N$ with $N = 2^m$:
\begin{equation}
    \mathrm{rank}\bigl(\mathbf{C}^u - \mathbf{H}^u\bigr) \;\leq\; N - m - 1
\end{equation}
and $\dim\ker(\mathbf{C}^u - \mathbf{H}^u) \geq m+1$. The null space contains the $m+1$ linearly independent vectors
\begin{equation}
    \mathbf{b}_0 \quad\text{and}\quad \mathbf{f}_p \;\triangleq\; \sum_{\substack{j=1 \\ \mathrm{lsb}(j)=p}}^{N-1} \mathbf{b}_j, \quad p = 0, 1, \ldots, m-1
\end{equation}
Here $\mathbf{b}_j \in \mathbb{R}^N$ is the $j$-th standard basis vector. Moreover, for generic $u$ (all $u \in \mathbb{R}^N$ outside a measure-zero algebraic set):
\begin{equation}
    \mathrm{rank}\bigl(\mathbf{C}^u - \mathbf{H}^u\bigr) = N - m - 1
    \qquad\text{and}\qquad
    \ker(\mathbf{C}^u - \mathbf{H}^u) = \mathrm{span}\{\mathbf{b}_0,\,\mathbf{f}_0,\ldots,\mathbf{f}_{m-1}\}
\end{equation}
We call $\mathrm{span}\{\mathbf{b}_0,\mathbf{f}_0,\ldots,\mathbf{f}_{m-1}\}$ the \emph{universal zero-error subspace}.
\end{theorem}

\noindent\textit{Proof sketch.}

\smallskip
\noindent\textit{Linear independence and kernel membership of $\mathbf{b}_0$.} The lsb-level sets $\{0\}$ and $\{j:\mathrm{lsb}(j)=p\}$ for $p=0,\ldots,m-1$ partition $\{0,\ldots,N-1\}$ and are the supports of $\mathbf{b}_0,\mathbf{f}_0,\ldots,\mathbf{f}_{m-1}$; disjoint supports imply linear independence. Since $(n-0)\bmod N = n\oplus 0$ for all $n$, every row of $\mathbf{C}^u-\mathbf{H}^u$ vanishes on $\mathbf{b}_0$.

\smallskip
\noindent\textit{Kernel membership of $\mathbf{f}_p$.} Write $n=n_0+n_h\cdot 2^p$, $K=2^{m-p}$; row $n$ of $(\mathbf{C}^u-\mathbf{H}^u)\mathbf{f}_p$ equals $\sum_{\substack{d\,\mathrm{odd}\\1\le d<K}}[u[(n-d\cdot 2^p)\bmod N]-u[n\oplus d\cdot 2^p]]$. Both maps $d\mapsto(n_h-d)\bmod K$ and $d\mapsto n_h\oplus d$ send odd $d$ to values of parity opposite to $n_h$ (subtracting or XOR-ing an odd integer flips the low-order bit), are injective as restrictions of full-domain bijections, and hence are equal-image bijections onto the same $K/2$-element set; the $u$-index multisets coincide and the row equals zero. Full details in Appendix~\ref{app:fp-kernel-proof}.

\smallskip
\noindent\textit{Generic rank $N-m-1$.} It suffices to show $G=\langle\rho_n : n=0,\ldots,N-1\rangle$ acts on $\{0,\ldots,N-1\}$ with exactly $m+1$ orbits (Appendix~\ref{app:rank-tightness}): $\mathrm{lsb}$ is $G$-invariant (a bit-arithmetic case analysis on $\mathrm{lsb}(n)$ vs.\ $\mathrm{lsb}(j)$ shows $\rho_n$ preserves the least significant $1$-bit of $j$), each lsb-level forms a single orbit (step-$2$ adjacency on odd integers connects the entire level via a suitable witness choice of $n$), and $\{0\}$ is a fixed orbit since $\rho_n(0)=n\oplus n=0$. By dimension, $\ker(\mathbf{C}^u-\mathbf{H}^u)=\mathrm{span}\{\mathbf{b}_0,\mathbf{f}_0,\ldots,\mathbf{f}_{m-1}\}$ for generic $u$. Full proofs of kernel membership and rank tightness are in Appendices~\ref{app:fp-kernel-proof} and~\ref{app:rank-tightness}.

\medskip
\noindent\textbf{Remark (Non-generic filters).} The rank $N-m-1$ is tight for generic $u$, but can collapse to zero: any $u\in\mathrm{span}\{\mathbf{b}_0,\mathbf{f}_0,\ldots,\mathbf{f}_{m-1}\}$ satisfies $\mathbf{C}^u=\mathbf{H}^u$ and hence $\mathrm{rank}(\mathbf{C}^u-\mathbf{H}^u)=0$ (Appendix~\ref{app:rank-tightness}). Concretely, for $N=4$ the filter $u=[0,1,0,1]^\top$ assigns equal weight to the odd indices $\{1,3\}$ (the unique lsb-$0$ orbit) and zero elsewhere; direct computation confirms $\mathbf{C}^u=\mathbf{H}^u$.

\subsection{Expected Error}
\label{sec:expected-error}

The third result gives a closed-form expression for the expected error magnitude for random filters and signals, governed by a single alignment scalar. Central to our analysis is the \emph{alignment scalar}:
\begin{equation}
    \label{eq:Q-def}
    Q(u) \;\triangleq\; \sum_{n,k} u\bigl[(n-k)\bmod N\bigr]\,u[n\oplus k]
\end{equation}
which measures the index agreement between $\mathbf{C}^u$ and $\mathbf{H}^u$. Since $(n-k)\bmod N$ and $n\oplus k$ are each bijections in $k$ for fixed $n$, both matrices satisfy $\|\mathbf{C}^u\|_F^2 = \|\mathbf{H}^u\|_F^2 = N\|u\|_2^2$; the standard expansion $\|A-B\|_F^2 = \|A\|_F^2 - 2\langle A,B\rangle_F + \|B\|_F^2$ then gives:
\begin{equation}
    \label{eq:frobenius}
    \|\mathbf{C}^u-\mathbf{H}^u\|_F^2 \;=\; 2N\|u\|_2^2 - 2Q(u)
\end{equation}
For fixed $u$ and i.i.d.\ $v$ with zero mean and variance $\sigma_v^2$, expanding:
\begin{equation}
    \label{eq:Av-expand}
    \mathbb{E}_v\bigl[\|(\mathbf{C}^u-\mathbf{H}^u)v\|_2^2\bigr] \;=\; \sum_{n}\sum_{k,\ell} (\mathbf{C}^u-\mathbf{H}^u)_{nk}\,(\mathbf{C}^u-\mathbf{H}^u)_{n\ell}\,\mathbb{E}[v[k]v[\ell]] \;=\; \|\mathbf{C}^u-\mathbf{H}^u\|_F^2\,\sigma_v^2
\end{equation}
where cross-terms $k\neq\ell$ vanish because $v[k]$ are independent with zero mean. Substituting gives:
\begin{equation}
    \label{eq:fixed-filter-error}
    \mathbb{E}_v\bigl[\|\mathbf{e}\|_2^2\bigr] \;=\; \bigl(2N\|u\|_2^2 - 2Q(u)\bigr)\sigma_v^2
\end{equation}

\medskip
\begin{theorem}[Expected squared error, i.i.d.\ inputs]\label{thm:expected-error}
Let $u[0],\ldots,u[N-1]$ be i.i.d.\ with zero mean and variance $\sigma_u^2$, let $v[0],\ldots,v[N-1]$ be i.i.d.\ with zero mean and variance $\sigma_v^2$, and let $u$ and $v$ be independent of each other. Then:
\begin{equation}
    \label{eq:expected-error}
    \mathbb{E}\bigl[\|\mathbf{e}\|_2^2\bigr] \;=\; 2\sigma_u^2\sigma_v^2\bigl(N^2 - 4\cdot 3^{m-1}\bigr)
\end{equation}
\end{theorem}

\noindent\textit{Proof.} Using equation~\eqref{eq:frobenius}, taking expectation over $u$: $\mathbb{E}[\|u\|_2^2] = N\sigma_u^2$, so $\mathbb{E}[2N\|u\|_2^2] = 2N^2\sigma_u^2$. For the cross-term, $\mathbb{E}[Q(u)] = M\sigma_u^2$ where $M = \sum_{n,k}\mathbb{1}[(n-k)\bmod N = n\oplus k] = \sum_{n=0}^{N-1}|S_n|$ counts matching index pairs: for a matching pair both terms index the same entry of $u$, giving $\mathbb{E}[u[(n-k)\bmod N]\,u[n\oplus k]] = \sigma_u^2$; for a non-matching pair the indices are distinct, so $\mathbb{E}[u[a]u[b]] = 0$ by independence.

By Theorem~\ref{thm:matching-set}, $|S_n| = 2^{\mathrm{popcount}(n')+1} = 2\cdot 2^{\mathrm{popcount}(n')}$, so:
\begin{align}
    M &= \sum_{n=0}^{N-1}|S_n|
      = 2\sum_{n=0}^{N-1}2^{\mathrm{popcount}(n\bmod N/2)} \notag\\
      &= 4\sum_{\nu=0}^{N/2-1}2^{\mathrm{popcount}(\nu)}
      = 4\cdot 3^{m-1}
\end{align}
where the third equality substitutes $\nu = n\bmod(N/2)$, folding each $\nu \in \{0,\ldots,N/2-1\}$ into its two preimages $n=\nu$ and $n=\nu+N/2$, and the fourth uses $\sum_{\nu=0}^{N/2-1}2^{\mathrm{popcount}(\nu)} = 3^{m-1}$ (the $m-1$ bit positions of $\nu$ each contribute a factor $1+2=3$). Substituting $M = 4\cdot 3^{m-1}$ into $\mathbb{E}[\|\mathbf{C}^u-\mathbf{H}^u\|_F^2] = 2\sigma_u^2(N^2-M)$ gives $\mathbb{E}[\|\mathbf{C}^u-\mathbf{H}^u\|_F^2] = 2\sigma_u^2(N^2 - 4\cdot 3^{m-1})$. Conditioning on $u$ and applying $\mathbb{E}_v[\|(\mathbf{C}^u-\mathbf{H}^u)v\|_2^2] = \|\mathbf{C}^u-\mathbf{H}^u\|_F^2\sigma_v^2$ yields equation~\eqref{eq:expected-error}. $\square$

\medskip
\begin{corollary}[Error-to-output energy ratio]\label{cor:ratio}
Under the i.i.d.\ model of Theorem~\ref{thm:expected-error},
\begin{equation}
    \label{eq:ratio}
    \frac{\mathbb{E}[\|\mathbf{e}\|_2^2]}{\mathbb{E}[\|y_{\mathrm{C}}\|_2^2]} \;=\; 2\!\left(1-\!\left(\tfrac{3}{4}\right)^{\!m-1}\right)
\end{equation}
which approaches $2$ as $N\to\infty$.
\end{corollary}

\noindent\textit{Proof.} The denominator follows by applying the same expansion as equation~\eqref{eq:Av-expand} to $\mathbf{C}^u$ alone: $\mathbb{E}_v[\|\mathbf{C}^u v\|_2^2] = \|\mathbf{C}^u\|_F^2\sigma_v^2 = N\|u\|_2^2\sigma_v^2$, and averaging over $u$ gives $\mathbb{E}[\|y_{\mathrm{C}}\|_2^2] = N\mathbb{E}[\|u\|_2^2]\sigma_v^2 = N^2\sigma_u^2\sigma_v^2$. Dividing equation~\eqref{eq:expected-error} by this and simplifying yields equation~\eqref{eq:ratio}. $\square$

\medskip
\noindent\textbf{Remark (Error energy exceeds output energy asymptotically).} For large $N$ the substitution error carries twice the energy of the intended output. The mechanism is decorrelation: $\mathbb{E}[\|\mathbf{e}\|_2^2] = 2\,\mathbb{E}[\|y_{\mathrm{C}}\|_2^2] - 2\,\mathbb{E}[\langle y_{\mathrm{C}}, y_{\mathrm{D}}\rangle]$, and the normalized cross-term $(3/4)^{m-1}$ decays to zero as $N\to\infty$ (Theorem~\ref{thm:expected-error}), so the two energies accumulate additively rather than cancelling.

\medskip
\noindent The alignment scalar $Q(u)$ is itself a quadratic form whose $N$-eigenspace is exactly the universal zero-error subspace of Section~\ref{sec:null-space}, tying the rank and alignment results together.

\medskip
\begin{theorem}[Range and eigenspace of the alignment scalar $Q(u)$]\label{thm:lambda-max}
For all $N = 2^m$ and $u \in \mathbb{R}^N$:
\begin{equation}
    -N\|u\|_2^2 \;\leq\; Q(u) \;\leq\; N\,\|u\|_2^2
\end{equation}
The upper bound is achieved exactly on the universal zero-error subspace $\mathrm{span}\{\mathbf{b}_0, \mathbf{f}_0, \ldots, \mathbf{f}_{m-1}\}$ of Theorem~\ref{thm:null-space}.
\end{theorem}

\noindent\textit{Proof.} Both bounds follow from Cauchy--Schwarz: $|Q(u)| = |\langle\mathbf{C}^u,\mathbf{H}^u\rangle_F| \leq \|\mathbf{C}^u\|_F\|\mathbf{H}^u\|_F = N\|u\|_2^2$. Upper bound and eigenspace characterization: Appendix~\ref{app:quadratic-form}. $\square$

\section{Discussion}

\paragraph{Matching-set geometry as the unifying structure.}
The fixed‑point structure of $\rho_n$ underlies all three classes of results: the matching set determines the zero‑error conditions (Sections~\ref{sec:foundational}–\ref{sec:general-structure}), the way $\rho_n$ partitions indices into orbits determines the null‑space dimension (Section~\ref{sec:null-space}), and the total number of matching pairs determines the closed‑form expression for the expected error (Section~\ref{sec:expected-error}). The two universally zero‑error output positions (Theorem~\ref{thm:zero-error}) and the two universally zero‑error input positions $k \in \{0,\,N/2\}$ (Corollary~\ref{cor:dual-input}) are those where every index matches. Furthermore, no output permutation can reconcile the two convolutions for arbitrary inputs (Theorem~\ref{thm:zero-error}). At all other positions, the error is a dot product between the input and a filter‑difference signal whose support equals the set of non‑matching indices; smaller support leads to smaller error.

\paragraph{Near-full rank and practical implications.}
The substitution error affects nearly all inputs: for generic filters, the error operator has rank $N-m-1$, leaving a zero‑error subspace of only $\log_2 N + 1$ dimensions. At $N=1024$, the zero-error subspace spans just $11$ of $1024$ dimensions ($1.1\%$); at $N=65536$, $17$ of $65536$ ($0.026\%$), shrinking rapidly with signal length. Hadamard layers proposed as drop-in replacements for DFT-based convolution~\cite{pan2021fast, pan2022block, pan2023htlayer} therefore compute a structurally distinct operation on all but a vanishingly small fraction of inputs. Combined with Corollary~\ref{cor:ratio}, this implies that the substitution error carries substantial energy for essentially all inputs. The normalized alignment $Q(u)/\|u\|_2^2$ (Section~\ref{sec:expected-error}) provides a single-number diagnostic: the closer this ratio is to its maximum value of $N$, the smaller the substitution error.

\paragraph{Filter alignment as the governing design variable.}
For random filters and signals, the error‑to‑output energy ratio approaches $2$ as signal length increases (Corollary~\ref{cor:ratio}): absent structural alignment, the substitution error asymptotically doubles the output energy. Filters $u$ lying in the zero-error subspace achieve maximum alignment ($Q(u)/\|u\|_2^2 = N$, the largest value permitted by Theorem~\ref{thm:lambda-max}) and incur no error, while the fraction of maximum alignment for generic i.i.d.\ filters satisfies $\mathbb{E}[Q(u)]/(N\cdot\mathbb{E}[\|u\|_2^2]) = \tfrac{4}{3}(3/4)^m \to 0$ as $N\to\infty$, placing them increasingly far from the zero-error regime. For a learned filter, the normalized alignment $Q(u)/\|u\|_2^2$ is bounded by $[-N,\,N]$ (Theorem~\ref{thm:lambda-max}), and whether training drives filters toward or away from the zero-error subspace is an open empirical question.

\section{Acknowledgements}

The authors thank Zahra Ronaghi and Saee Paliwal for supporting this project. The authors also thank Brian L. Evans of The University of Texas at Austin and Dilip Sarwate of the University of Illinois Urbana-Champaign for discussions that motivated this work. The authors declare no financial interest or conflicts.

\bibliographystyle{plainnat}
\bibliography{wht_conv}

\newpage

\setcounter{theorem}{0}
\appendix

\section{Proof of Theorem~\ref{thm:matching-set} (Matching Set Structure)}
\label{app:matching-set-proof}

\noindent\textit{Proof.} Write $n = n' + c_n \cdot N/2$ and $j = j' + c_j \cdot N/2$ with $n', j' \in \{0,\ldots,N/2-1\}$ and $c_n, c_j \in \{0,1\}$. In both cases ($c_n = c_j$ and $c_n \neq c_j$), the leading-bit contributions cancel: $n \oplus j = (n' \oplus j') + (c_n \oplus c_j)\cdot N/2$, and $(n-j)\bmod N$ evaluates as follows.

\medskip
\noindent\textbf{Case A} ($c_n = c_j$): $n - j = n' - j'$. For $j' \leq n'$: $(n-j)\bmod N = n'-j'$ (non-negative, less than $N/2 \leq N$, no reduction needed), and $n \oplus j = n' \oplus j'$ (leading bits cancel). For $j' > n'$: $(n-j)\bmod N = N + (n'-j') \geq N/2+1$, so the leading bit is $1$; but $n \oplus j = n' \oplus j' < N/2$ has leading bit $0$. No agreement.

\medskip
\noindent\textbf{Case B} ($c_n \neq c_j$): Sub-case $c_n=1, c_j=0$: $n-j = N/2 + n' - j' \in [1,N-1]$, no reduction is needed. Sub-case $c_n=0, c_j=1$: $n-j = n' - N/2 - j' < 0$, adding $N$ gives $N/2 + n' - j'$. In both sub-cases: $(n-j)\bmod N = N/2 + (n'-j')$ and $n \oplus j = (n' \oplus j') + N/2$ (since $c_n \oplus c_j = 1$). These agree iff $n'-j' = n' \oplus j'$.

\medskip
In all cases, the condition reduces to $n' - j' = n' \oplus j'$ with $j' \leq n'$. The identity $a - b = a \oplus b$ holds for non-negative integers iff $b$ is a submask of $a$ (the subtraction has no borrows). Therefore $(n-j)\bmod N = n \oplus j$ if and only if $j' \subseteq n'$.

The number of $j \in \{0,\ldots,N-1\}$ with $j' = j \bmod N/2 \subseteq n'$ is: $2^{\mathrm{popcount}(n')}$ choices for $j'$ (one per submask of $n'$) times $2$ choices for $c_j$, giving $|S_n| = 2^{\mathrm{popcount}(n')+1}$. $\square$

\section{Proof of Kernel Membership of $\mathbf{f}_p$ (Theorem~\ref{thm:null-space})}
\label{app:fp-kernel-proof}

\noindent\textit{Proof.} Row $n$ of $(\mathbf{C}^u-\mathbf{H}^u)\mathbf{f}_p$ equals:
\[
    \sum_{\substack{d \text{ odd} \\ 1 \leq d < 2^{m-p}}} \Bigl[u\bigl[(n - d\cdot 2^p)\bmod N\bigr] - u\bigl[n \oplus d\cdot 2^p\bigr]\Bigr]
\]
Write $n = n_0 + n_h\cdot 2^p$ with $n_0 = n\bmod 2^p$ and $n_h = \lfloor n/2^p\rfloor$. Since $d\cdot 2^p$ is a multiple of $2^p$, the low $p$ bits are unaffected by either operation:
\begin{align*}
    (n - d\cdot 2^p)\bmod N &= n_0 + \bigl((n_h - d)\bmod 2^{m-p}\bigr)\cdot 2^p \\
    n \oplus d\cdot 2^p &= n_0 + (n_h \oplus d)\cdot 2^p
\end{align*}
Let $K = 2^{m-p}$ and $P = \{\ell \in \{0,\ldots,K-1\} : \ell \not\equiv n_h \pmod{2}\}$; note $|P| = K/2$. As $d$ ranges over the $K/2$ odd values in $\{1,\ldots,K-1\}$:
\begin{itemize}
\item \emph{Image in $P$.} Subtracting (resp.\ XOR-ing) an odd $d$ flips the low-order bit of $n_h$, so $(n_h - d)\bmod K$ and $n_h \oplus d$ both have parity opposite to $n_h$; hence both maps send odd $d$ into $P$.
\item \emph{Injectivity on odd $d$.} The map $d\mapsto(n_h-d)\bmod K$ is a bijection on all of $\{0,\ldots,K-1\}$; a bijection on the full domain restricts to an injection on any subset. Similarly, $d\mapsto n_h\oplus d$ is a bijection on $\{0,\ldots,K-1\}$, hence injective on odd $d$.
\item \emph{Bijection onto $P$.} Each restricted map is an injection from a $K/2$-element domain into a $K/2$-element codomain $P$, so both restrictions are bijections onto $P$.
\end{itemize}
Since both maps have the same image $P$, the multisets $\{(n_h-d)\bmod K : d\text{ odd}\}$ and $\{n_h\oplus d : d\text{ odd}\}$ are equal. The corresponding $u$-indices $n_0 + \ell\cdot 2^p$ therefore coincide as multisets, the sums over $u[\cdot]$ are equal, and the difference is zero. $\square$

\section{Proof of Rank Tightness for Generic $u$ (Theorem~\ref{thm:null-space})}
\label{app:rank-tightness}

\noindent\textit{Proof.} We work over $\mathbb{Q}(u_0,\ldots,u_{N-1})$, treating the filter entries as algebraically independent indeterminates; the Conclusion converts this to the measure-theoretic statement.

\medskip
\noindent\textit{Step 1: Generic null space = functions constant on orbits.} Over $\mathbb{Q}(u_0,\ldots,u_{N-1})$, $(\mathbf{C}^u - \mathbf{H}^u)\mathbf{w} = 0$ requires:
\[
\sum_a u[a]\bigl(\mathbf{w}[(n-a)\bmod N] - \mathbf{w}[n\oplus a]\bigr) = 0 \quad\text{for all }n
\]
Let $q \in \mathbb{Q}[u_0,\ldots,u_{N-1}]$ be a common denominator for the entries $\mathbf{w}[j] \in \mathbb{Q}(u_0,\ldots,u_{N-1})$, and write $\mathbf{w}[j] = g[j]/q$ with $g[j] \in \mathbb{Q}[u_0,\ldots,u_{N-1}]$. Multiplying through by $q$ gives $\sum_a u[a](g[(n-a)\bmod N] - g[n\oplus a]) = 0$ as a polynomial identity; since $u_0,\ldots,u_{N-1}$ are algebraically independent, each coefficient vanishes: $g[(n-a)\bmod N] = g[n\oplus a]$ for all $n,a$, hence $\mathbf{w}[(n-a)\bmod N] = \mathbf{w}[n\oplus a]$ for all $n,a$. Substituting $b = (n-a)\bmod N$ (so $n\oplus a = \rho_n(b)$), this reads $\mathbf{w}[b] = \mathbf{w}[\rho_n(b)]$ for all $n,b$. That is, the null space over $\mathbb{Q}(u_0,\ldots,u_{N-1})$ equals the space of vectors constant on orbits of $G = \langle\rho_n : n = 0,\ldots,N-1\rangle$.

\medskip
\noindent\textit{Step 2: Orbit Lemma.} $G$ acts on $\{0,\ldots,N-1\}$ with exactly $m+1$ orbits: $\{0\}$ and $\{j : \mathrm{lsb}(j) = p\}$ for $p = 0,\ldots,m-1$. First, $\rho_n(0) = n \oplus ((n-0)\bmod N) = n \oplus n = 0$ for all $n$, so $\{0\}$ is fixed by every generator and hence forms a $G$-orbit on its own.

\smallskip
\emph{(a) lsb is $G$-invariant.} We show $\mathrm{lsb}(\rho_n(j)) = \mathrm{lsb}(j)$ for all $n,j$. Set $s = \mathrm{lsb}(n)$, $t = \mathrm{lsb}(j)$, write $n = 2^s \bar{n}$, $j = 2^t \bar{j}$ with $\bar{n},\bar{j}$ odd (here $\bar{\cdot}$ denotes the odd part after factoring out the power of $2$; distinct from $n' = n\bmod N/2$ used in the main body and Appendix~\ref{app:matching-set-proof}), and let $\alpha = (n-j)\bmod N$.

\begin{itemize}
\item $s > t$: Then $\alpha = 2^t \cdot (\text{odd})$, so bit $t$ of $\alpha$ is $1$. Since $s > t$, bit $t$ of $n$ is $0$. Thus bit $t$ of $n\oplus \alpha = 0\oplus 1 = 1$, and bits $0,\ldots,t-1$ of $n\oplus \alpha$ are zero. Hence $\mathrm{lsb}(\rho_n(j)) = t$.

\item $s < t$: Write $n - j = 2^s(\bar{n} - 2^{t-s}\bar{j})$. Since $\bar{n}$ is odd and $2^{t-s}\bar{j}$ is even, the factor $\bar{n} - 2^{t-s}\bar{j}$ is odd; set $c = (\bar{n} - 2^{t-s}\bar{j})\bmod 2^{m-s}$ (odd), so $\alpha = 2^s c$. Since both $n$ and $\alpha$ are divisible by $2^s$, we have $n\oplus \alpha = 2^s(\bar{n}\oplus c)$, so $\mathrm{lsb}(\rho_n(j)) = s + \mathrm{lsb}(\bar{n}\oplus c)$. Bits $0,\ldots,t-s-1$ of $c$ equal those of $\bar{n}$: since $2^{t-s}\bar{j}$ is zero in those bit positions and binary subtraction carries propagate toward higher-order bits, zeros in positions $0,\ldots,t-s-1$ of $2^{t-s}\bar{j}$ guarantee those bit positions of $c$ match those of $\bar{n}$. Hence bits $0,\ldots,t-s-1$ of $\bar{n}\oplus c$ are zero. At bit $t-s$: since $\bar{j}$ is odd, bit $t-s$ of $2^{t-s}\bar{j}$ is $1$, and there is no borrow from below, so subtracting flips bit $t-s$ of $\bar{n}$; hence the $(t-s)$-th bit of $\bar{n}\oplus c$ is $1$. Therefore $\mathrm{lsb}(\bar{n}\oplus c) = t-s$ and $\mathrm{lsb}(\rho_n(j)) = s+(t-s) = t$.

\item $s = t$: Then $\alpha = 2^t(\text{even})$, so $\mathrm{lsb}(\alpha) > t$, bit $t$ of $\alpha = 0$, bit $t$ of $n = 1$. Thus, bit $t$ of $n\oplus \alpha = 1$ and bits $0,\ldots,t-1$ are zero. Hence $\mathrm{lsb}(\rho_n(j)) = t$.
\end{itemize}

\smallskip
\emph{Carry-theoretic reading.} The case analysis has a compact restatement: the first borrow in the subtraction $n-j$ occurs at bit $t=\mathrm{lsb}(j)$, since bits $0,\ldots,t-1$ of $j$ are zero and bit $t$ of $j$ is $1$. Bits $0,\ldots,t-1$ of $\alpha=(n-j)\bmod N$ therefore agree with those of $n$, and bit $t$ differs, giving $\mathrm{lsb}(n\oplus \alpha)=t$ directly. Correspondingly, $j\in S_n$ iff the subtraction of their lower $m-1$ bits is borrow-free---equivalently, iff $(j\bmod N/2)$ is a submask of $(n\bmod N/2)$ (Theorem~\ref{thm:matching-set}).

\smallskip
\emph{(b) Transitivity.} We show $G$ acts transitively on $\{j : \mathrm{lsb}(j) = p\}$ for each $p$. Writing $n = 2^p n_h + n_0$ with $n_0 = n\bmod 2^p$, the action of $\rho_n$ on $2^p\tilde{j}$ (odd $\tilde{j}$) is:
\[
\rho_n(2^p \tilde{j}) = 2^p\bigl(n_h \oplus ((n_h - \tilde{j})\bmod K)\bigr) = 2^p\,\rho^{(K)}_{n_h}(\tilde{j})
\]
where $\rho^{(K)}_a(b) = a\oplus((a-b)\bmod K)$ and $K = 2^{m-p}$. Since $d\cdot 2^p$ is divisible by $2^p$, the low $p$ bits $n_0$ are unaffected by both subtraction and XOR, and cancel from $\rho_n(2^p\tilde{j})$. It suffices to show $G^{(K)} = \langle\rho^{(K)}_{n_h}\rangle$ acts transitively on the set $O_K$ of odd elements of $\{0,\ldots,K-1\}$. For $K=2$ (i.e., $p=m-1$), $O_K = \{1\}$ is a single element and transitivity is immediate. For $K \geq 4$, for any odd $\tilde{j} \leq K-3$, choose $n_h = \tilde{j}+1$ (even, so $n_h - \tilde{j} = 1$):
\[
\rho^{(K)}_{\tilde{j}+1}(\tilde{j}) \;=\; (\tilde{j}+1)\oplus 1 \;=\; \tilde{j}+2
\]
since $\tilde{j}+1$ is even (bit 0 is $0$), so XOR with $1$ sets bit 0, giving $\tilde{j}+2$. Thus, for every odd $\tilde{j} \leq K-3$, the elements $\tilde{j}$ and $\tilde{j}+2$ lie in the same $G^{(K)}$-orbit. Since $O_K = \{1,3,5,\ldots,K-1\}$ and the step-$2$ adjacency connects $1\sim 3\sim 5\sim\cdots\sim K-1$, transitivity of the orbit equivalence relation implies all elements of $O_K$ share a single orbit. Hence $G^{(K)}$ acts transitively on $O_K$.

\medskip
\noindent\textit{Conclusion.} Steps~1 and~2 together show that the null space of $\mathbf{C}^u-\mathbf{H}^u$ over $\mathbb{Q}(u_0,\ldots,u_{N-1})$ is exactly the space of functions constant on the $m+1$ orbits of $G$, which has dimension $m+1$; hence $\mathrm{rank}(\mathbf{C}^u-\mathbf{H}^u) = N-(m+1) = N-m-1$ over $\mathbb{Q}(u_0,\ldots,u_{N-1})$. By definition of rank over a field, some $(N-m-1)\times(N-m-1)$ submatrix has nonzero determinant in $\mathbb{Q}(u_0,\ldots,u_{N-1})$; since every entry of $\mathbf{C}^u-\mathbf{H}^u$ is linear in $u_0,\ldots,u_{N-1}$, this determinant is a nonzero polynomial $P\in\mathbb{Z}[u_0,\ldots,u_{N-1}]$ (not merely a nonzero rational function). Since the entries of $\mathbf{C}^u-\mathbf{H}^u$ are linear in $u$, every minor is a polynomial, so $\{u:\mathrm{rank}\geq N-m-1\}$ is Zariski open. A nonzero polynomial over $\mathbb{R}$ cannot vanish on all of $\mathbb{R}^N$, so $P$ takes a nonzero value at some $u\in\mathbb{R}^N$ and the Zariski-open set is nonempty. Its complement is therefore a proper Zariski-closed set in $\mathbb{R}^N$, which has Lebesgue measure zero, so $\mathrm{rank}(\mathbf{C}^u-\mathbf{H}^u) = N-m-1$ for generic $u\in\mathbb{R}^N$. Finally, by Theorem~\ref{thm:null-space}, $\mathrm{span}\{\mathbf{b}_0, \mathbf{f}_0, \ldots, \mathbf{f}_{m-1}\} \subseteq \ker(\mathbf{C}^u - \mathbf{H}^u)$ for every $u$; for generic $u$, $\dim\ker = m+1 = \dim\,\mathrm{span}\{\mathbf{b}_0, \mathbf{f}_0, \ldots, \mathbf{f}_{m-1}\}$, so the inclusion is an equality. $\square$

\medskip
\noindent\textbf{Remark (Numerical illustration; not part of the proof).} The computations below verify the rank result of Theorem~\ref{thm:null-space} for the witness filter $u^* = [0,1,\ldots,N-1]^\top$, for which $(\mathbf{C}^{u^*}-\mathbf{H}^{u^*})_{n,k} = (n-k)\bmod N - (n\oplus k)$; they are provided for illustrative purposes only.

\smallskip
\noindent\textit{Case $N=4$, $m=2$.} The matrix $\mathbf{C}^{u^*}-\mathbf{H}^{u^*}$ is:
\[
\begin{pmatrix}
 0 &  2 & 0 & -2 \\
 0 &  0 & 0 &  0 \\
 0 & -2 & 0 &  2 \\
 0 &  0 & 0 &  0
\end{pmatrix}
\]
Rows $n=1$ and $n=3$ are zero (the universal zero-error output positions $N/2-1$ and $N-1$); row~$0$ and row~$2$ sum to zero. Hence $\mathrm{rank}=1=N-m-1$, and the null space is $\mathrm{span}\{\mathbf{b}_0,\mathbf{f}_0,\mathbf{f}_1\}$ where $\mathbf{b}_0=[1,0,0,0]^\top$, $\mathbf{f}_0=[0,1,0,1]^\top$ (odd indices), $\mathbf{f}_1=[0,0,1,0]^\top$ (index~$2$).

\smallskip
\noindent\textit{Case $N=8$, $m=3$.} The eight rows of $\mathbf{C}^{u^*}-\mathbf{H}^{u^*}$ are:
\begin{align*}
r_0 &= [\phantom{-}0,\phantom{-}6,\phantom{-}4,\phantom{-}2,\phantom{-}0,-2,-4,-6],\\
r_1 &= [\phantom{-}0,\phantom{-}0,\phantom{-}4,\phantom{-}4,\phantom{-}0,\phantom{-}0,-4,-4],\\
r_2 &= [\phantom{-}0,-2,\phantom{-}0,\phantom{-}6,\phantom{-}0,-2,\phantom{-}0,-2],\\
r_3 &= [\phantom{-}0,\phantom{-}0,\phantom{-}0,\phantom{-}0,\phantom{-}0,\phantom{-}0,\phantom{-}0,\phantom{-}0],\\
r_4 &= [\phantom{-}0,-2,-4,-6,\phantom{-}0,\phantom{-}6,\phantom{-}4,\phantom{-}2],\\
r_5 &= [\phantom{-}0,\phantom{-}0,-4,-4,\phantom{-}0,\phantom{-}0,\phantom{-}4,\phantom{-}4],\\
r_6 &= [\phantom{-}0,-2,\phantom{-}0,-2,\phantom{-}0,-2,\phantom{-}0,\phantom{-}6],\\
r_7 &= [\phantom{-}0,\phantom{-}0,\phantom{-}0,\phantom{-}0,\phantom{-}0,\phantom{-}0,\phantom{-}0,\phantom{-}0].
\end{align*}
Rows $r_3$ and $r_7$ are zero (the universal zero-error positions $n=N/2-1$ and $n=N-1$). The six nonzero rows satisfy exactly two linear dependencies: $r_1+r_5=0$ and $r_0+r_2+r_4+r_6=0$. Hence $\mathrm{rank}=6-2=4=N-m-1$, and the null space is $\mathrm{span}\{\mathbf{b}_0,\mathbf{f}_0,\mathbf{f}_1,\mathbf{f}_2\}$ where $\mathbf{b}_0=\mathbf{e}_0$, $\mathbf{f}_0=\mathbf{e}_1+\mathbf{e}_3+\mathbf{e}_5+\mathbf{e}_7$ (odd indices), $\mathbf{f}_1=\mathbf{e}_2+\mathbf{e}_6$ (lsb-$1$ indices), $\mathbf{f}_2=\mathbf{e}_4$ (lsb-$2$ index).

\section{Quadratic Form Analysis: Proof of Theorem~\ref{thm:lambda-max}}
\label{app:quadratic-form}

Define the matrix $T \in \mathbb{Z}_{\geq 0}^{N \times N}$ by
\begin{equation}\label{eq:T-def}
    T[a,b] \;=\; \bigl|\bigl\{n \in \{0,\ldots,N-1\} : \rho_n(a) = b\bigr\}\bigr|
    \;=\; \bigl|\bigl\{n : n \oplus \bigl((n-a)\bmod N\bigr) = b\bigr\}\bigr|
\end{equation}
Then $Q(u) = u^\top T u = u^\top T_{\mathrm{sym}} u$ for all $u \in \mathbb{R}^N$, where $T_{\mathrm{sym}} = (T + T^\top)/2$ is symmetric.

\medskip
\noindent\textbf{Derivation (Quadratic-form representation).}
Starting from equation~\eqref{eq:Q-def}, fix $n$ and substitute $a = (n-k)\bmod N$ (a bijection in $k$). Then $k = (n-a)\bmod N$ and $n\oplus k = n\oplus((n-a)\bmod N) = \rho_n(a)$ (equation~\eqref{eq:rho-def}). Therefore:
\[
    Q(u) \;=\; \sum_n \sum_a u[a]\,u[\rho_n(a)]
    \;=\; \sum_{a,b} u[a]\,u[b]\,\bigl|\{n : \rho_n(a) = b\}\bigr|
    \;=\; u^\top T u
\]
with $T[a,b]$ as defined in equation~\eqref{eq:T-def}. $\square$

\medskip
\noindent\textbf{Properties of $T$ (used in Theorem~\ref{thm:lambda-max}).}

\textit{Row sums $= N$:} Since $\rho_n(a)$ is well-defined, for each $n$ it takes exactly one value $b$, so $\sum_b T[a,b] = \sum_b |\{n : \rho_n(a) = b\}| = N$.

\textit{Diagonal $T[a,a] = Q(\mathbf{b}_a)$:} $T[a,a] = |\{n : \rho_n(a) = a\}| = |\{n : a \in S_n\}| = Q(\mathbf{b}_a)$. By Theorem~\ref{thm:matching-set}, $a \in S_n$ iff $a \bmod N/2$ is a submask of $n \bmod N/2$, giving $T[a,a] = 2^{m-\mathrm{popcount}(a\bmod N/2)}$; hence $Q(\mathbf{b}_a)$ is available in $O(1)$ via popcount without materializing $T$.

\textit{Precomputation cost:} $T$ can be precomputed in $O(N^2)$ time and space by iterating over all $(n,a)$ pairs; thereafter $Q(u) = u^\top Tu$ costs $O(N^2)$ per filter.

\textit{The symmetry of $Q(u)$:} Since $u^\top A u = 0$ for any antisymmetric matrix $A$, we have $Q(u) = u^\top T_{\mathrm{sym}} u$ for all $u$, so all eigenvalue analysis uses $T_{\mathrm{sym}} = (T+T^\top)/2$. Column sums of $T$ also equal $N$ (for fixed $b$ and each $n$, $\rho_n$ is a bijection so exactly one $a$ maps to $b$; summing gives $\sum_a T[a,b] = N$), so $T^\top \mathbf{1} = N\mathbf{1}$ and $T_{\mathrm{sym}}\mathbf{1} = N\mathbf{1}$.

\medskip
\noindent\textbf{Proof of Theorem~\ref{thm:lambda-max}.}
The upper bound $Q(u) \leq N\|u\|_2^2$ follows from Cauchy--Schwarz (proved in the main body); equivalently, $\lambda_{\max}(T_{\mathrm{sym}}) \leq N$. It remains to show the bound is achieved exactly on $\mathrm{span}\{\mathbf{b}_0, \mathbf{f}_0, \ldots, \mathbf{f}_{m-1}\}$. Since $(\mathbf{C}^{\mathbf{b}_0})_{n,k} = (\mathbf{b}_0)[(n-k)\bmod N] = \mathbf{1}[n=k]$ and $(\mathbf{H}^{\mathbf{b}_0})_{n,k} = (\mathbf{b}_0)[n\oplus k] = \mathbf{1}[n=k]$, both matrices equal $I_N$ by definition, so $\|\mathbf{C}^{\mathbf{b}_0}-\mathbf{H}^{\mathbf{b}_0}\|_F^2 = 0$; equation~\eqref{eq:frobenius} then gives $Q(\mathbf{b}_0) = N\|\mathbf{b}_0\|_2^2$, so $T_{\mathrm{sym}}\mathbf{b}_0 = N\mathbf{b}_0$ and $\lambda_{\max}(T_{\mathrm{sym}}) = N$. For each null vector $\mathbf{f}_p$: since lsb-invariance (Appendix~\ref{app:rank-tightness}) gives $\mathrm{lsb}(\rho_n(j)) = \mathrm{lsb}(j)$ for all $n,j$, and $\rho_n((n-k)\bmod N) = n\oplus k$, we have $\mathrm{lsb}((n-k)\bmod N) = \mathrm{lsb}(n\oplus k)$ for all $n,k$; hence $(\mathbf{C}^{\mathbf{f}_p})_{n,k} = \mathbf{f}_p[(n-k)\bmod N] = \mathbf{f}_p[n\oplus k] = (\mathbf{H}^{\mathbf{f}_p})_{n,k}$, so $\mathbf{C}^{\mathbf{f}_p} = \mathbf{H}^{\mathbf{f}_p}$ and $\|\mathbf{C}^{\mathbf{f}_p}-\mathbf{H}^{\mathbf{f}_p}\|_F^2 = 0$; equation~\eqref{eq:frobenius} then gives $Q(\mathbf{f}_p) = N\|\mathbf{f}_p\|_2^2$, so $\mathbf{f}_p$ lies in the $N$-eigenspace of $T_{\mathrm{sym}}$.

\medskip
The $N$-eigenspace of $T_{\mathrm{sym}}$ equals $\mathrm{span}\{\mathbf{b}_0, \mathbf{f}_0, \ldots, \mathbf{f}_{m-1}\}$ with dimension exactly $m+1$: $T_{\mathrm{sym}} u = Nu$ iff $Q(u) = N\|u\|_2^2$ iff $\|\mathbf{C}^u - \mathbf{H}^u\|_F^2 = 0$ (equation~\eqref{eq:frobenius}) iff $u[(n-k)\bmod N] = u[n\oplus k]$ for all $n,k$---i.e., iff $u$ is constant on orbits of $G = \langle\rho_n : n=0,\ldots,N-1\rangle$. The Orbit Lemma (Appendix~\ref{app:rank-tightness}) establishes that $G$ has exactly $m+1$ orbits ($\{0\}$ and $\{j:\mathrm{lsb}(j)=p\}$ for $p=0,\ldots,m-1$), so the space of functions constant on $G$-orbits has dimension $m+1$. Since each $\mathbf{b}_0, \mathbf{f}_0, \ldots, \mathbf{f}_{m-1}$ is constant on a single $G$-orbit and zero elsewhere, every element of $\mathrm{span}\{\mathbf{b}_0, \mathbf{f}_0, \ldots, \mathbf{f}_{m-1}\}$ is constant on $G$-orbits; as this span is $(m+1)$-dimensional, it equals the full $N$-eigenspace. $\square$

\end{document}